\documentclass{article}

\usepackage{mathptmx}
\usepackage{url}
\usepackage{amsthm}

\theoremstyle{definition}
\newtheorem{definition}{Definition}
\theoremstyle{plain}

\theoremstyle{plain}
\newtheorem{corollary}{Corollary}
\theoremstyle{plain}
\newtheorem{conjecture}{Conjecture}
\theoremstyle{remark}
\newtheorem{remark}{Remark}
\newtheorem{fact}{Fact}
\theoremstyle{plain}
\newtheorem{proposition}{Proposition}
\theoremstyle{plain}

\theoremstyle{definition}
\newtheorem{example}{Example}

\usepackage{caption}
\usepackage{todonotes}
\usepackage{lipsum,framed}
\usepackage{amsmath}
\usepackage{amssymb}
\usepackage{pifont}% http://ctan.org/pkg/pifont
\usepackage{comment}
\usepackage{tikz}
\usetikzlibrary{positioning,arrows,calc,shapes,babel,fit} 
\usetikzlibrary{decorations.pathmorphing}
\tikzset{
modal/.style={>=stealth',shorten >=1pt,shorten <=1pt,auto,node distance=1.5cm,
semithick},
world/.style={circle,draw,minimum size=0.5cm},
arg/.style={circle,draw,minimum size=0.5cm},
uncarg/.style={circle,draw,dashed,minimum size=0.5cm},
sarg/.style={draw},
carg/.style={draw,minimum size=0.5cm},
point/.style={circle,draw,inner sep=0.5mm,fill=black},
reflexive above/.style={->,loop,looseness=7,in=120,out=60},
reflexive below/.style={->,loop,looseness=7,in=240,out=300},
reflexive left/.style={->,loop,looseness=7,in=150,out=210},
reflexive right/.style={->,loop,looseness=5,in=30,out=330},
coil/.style={decorate, decoration={coil,amplitude=4pt,segment length=5pt}},
snake/.style={decorate, decoration={snake}},
zigzag/.style={decorate, decoration={zigzag}}
}
\newcommand{\ang}[1]{\langle #1 \rangle}

\newcommand{\uni}{\mathsf{U}}
\newcommand{\args}{\mathsf{A}}
\newcommand{\rel}{\mathsf{D}}
\newcommand{\lexpressive}{\preccurlyeq}

\newcommand{\compfun}{\mathsf{com}}

\newcommand{\ndefiaf}{\mathsf{IA}}

\newcommand{\comp}{\ang{\args,\defrel^{\ast}}}
\newcommand{\naf}{\mathsf{AF}}

\newcommand{\impor}{\mathsf{IMPLY}^{\lor}}
\newcommand{\opimp}{\mathsf{IMPLY}}
\newcommand{\opor}{\mathsf{OR}}
\newcommand{\opnand}{\mathsf{NAND}}

\newcommand{\ndepdefiaf}{\mathsf{dIA}}
\newcommand{\fix}{F}

\newcommand{\black}{\color{black}}

\newcommand{\lanset}{\mathsf{L}}
\newcommand{\rulset}{\mathsf{R}}
\newcommand{\namefun}{\mathfrak{n}}

\newcommand{\defrel}{\mathsf{D}}
\newcommand{\negfun}{\overline{\cdot}}

\newcommand{\at}{\mathsf{AT}}

\newcommand{\unravelat}{\ang{\lanset, \negfun, \rulset, \namefun,\kb}}

\newcommand{\kb}{\mathsf{K}}

\newcommand{\sub}{\mathsf{Sub}}
\newcommand{\prem}{\mathsf{Prem}}
\newcommand{\conc}{\mathsf{Conc}}
\newcommand{\ftoprule}{\mathsf{TopRule}}
\newcommand{\sto}{\!\twoheadrightarrow\!}

\newcommand{\prefprefix}{\mathsf{p}\text{-}}

\newcommand{\nsaf}{\mathsf{S}}
\newcommand{\daf}{\mathsf{AF}}

\newcommand{\nisaf}{\mathsf{IS}}

\newcommand{\unravelsaf}{\ang{\lanset, \negfun, \rulset, \namefun, \kb, \leq}}
\newcommand{\indexedsaf}[1]{\ang{\lanset_{#1}, \negfun_{#1}, \rulset_{#1}, \namefun_{#1}, \kb_{#1}, \leq_{#1}}}

\newcommand{\contfun}{\overline{\cdot}}

\newcommand{\pia}{\mathsf{PIA}}
\newcommand{\pda}{\mathsf{PDA}}

\begin{document}

\pagestyle{headings}
\def\thepage{}
%\begin{frontmatter}              % The preamble begins here.

%\pretitle{Pretitle}
\title{Some Results about the Expressivity of Preference-Incomplete Structured Argumentation Frameworks}
\author{Antonio Yuste-Ginel\\
University of M\'alaga \\
ayusteginel@uma.es}
\maketitle
%\markboth{}{April 2026\hb}
%\subtitle{Subtitle}

%

%\affiliation{University of M\'alaga}

\begin{abstract}
This paper studies the expressive power of ASPIC$^+$ argumentation frameworks with uncertain preference profiles by comparing them with several abstract formalisms with uncertain defeats. Most of our results are negative (and some of them are theoretically unexpected). We also conjecture a positive, non-trivial threshold for the expressivity of uncertain preferences, and prove some essential preliminary steps toward the confirmation of this conjecture.
\end{abstract}

% \begin{keyword}
% ASPIC \sep incomplete argumentation frameworks \sep preferences
% \end{keyword}
% \end{frontmatter}
%\markboth{April 2026\hb}{April 2026\hb}
%\thispagestyle{empty}
%\pagestyle{empty}

\sloppy
%==================
\section{Introduction}
%==================

\paragraph{Context} Arguing and uncertainty appear intertwined in many real contexts. On one hand, assessing conflicting arguments is a canonical way to deal with uncertain and incompatible sources of information. On the other hand, uncertain argumentative information (e.g, which rules or premises are available) is key to knowing which arguments are relevant and, eventually, which propositions shall get accepted. This second sense of influence is especially relevant in multi-agent scenarios (including adversarial ones), where the lack of information of one agent about another's argumentative background determines which arguments are considered as potentially good moves in dialogue. Consequently, the formal argumentation community has devoted considerable effort to modelling uncertainty in abstract argumentation frameworks, both in a probabilistic fashion \cite{hunter2021survey} and in a qualitative (i.e., non-numerical) one \cite{mailly2022yes}. Regarding the latter, the standard abstract tool is Incomplete Argumentation Frameworks (IAFs), these are Dung's frameworks \cite{dung1995acceptability}, where the sets of arguments and defeats are split into two disjoint components representing, respectively, certain and uncertain arguments and defeats. Reasoning about argument acceptability in IAFs is done via the notion of completion (hypothetical removals of uncertainty that can be seen as epistemically possible worlds).

\paragraph{Motivation and previous work} However, critical voices have emerged about the conceptual plausibility of remaining at an abstract formal level of argumentation, ignoring the internal structure of argument and the nature of conflict among them. In \cite{prakken2018abstraction}, the authors show how some popular abstract formalisms make strong assumptions that are unjustifiable once the structure of arguments is taken into account. This motivated a new line of research on the viability of instantiating abstract formalisms within the standard tools (e.g., ASPIC$^+$ \cite{modgil2013aspic} or ABA \cite{bondarenko1997aba}, among others). Such a trend has also permeated the analysis of arguing with qualitative uncertainty. In this line, structured argumentation with uncertain rules \cite{ai32023} and uncertain premises \cite{odekerken2025argumentative} have been studied (and recently compared \cite{proietti2025comparative}). Furthermore, structured frameworks with uncertain preferences were first mentioned in \cite{baumeister2021acceptance} and later defined in \cite{ria2024}, but a more comprehensive analysis remains missing. 

\paragraph{Focus and methodology} Straightforwardly, uncertain preferences generate uncertain defeats at the abstract level.\footnote{In structured argumentation frameworks, preferences are used to disregard some potential defeats (\cite{modgil2013aspic}).} An important research question is, then, what kind of abstract formalism do we need to capture the sort of uncertainty modelled by uncertain preferences? It has already been shown \cite[Proposition 2]{ria2024} that standard defeat-IAFs --IAFs where the only uncertain components are defeats (see, e.g., \cite{cayrol2007partial,att-IAFs-2015-acceptance})-- are not expressive enough: uncertain preferences might generate sets of completions at the abstract level that are not isomorphic to the set of completions of any defeat-IAF. This paper continues such an inquiry by comparing uncertain preferences with other abstract formalisms for defeat-uncertainty (especially, those of \cite{fazzingakr21}). Our reference framework for structured argumentation is ASPIC$^+$.

%\paragraph{Structure and contribution} The paper proceed as follows. 

%===================================================
\section{Background}\label{sec:background}
%===================================================
\subsection{Abstract argumentation frameworks with uncertain defeats}

In what follows, we assume as given a \textbf{background set of arguments} (names) $\uni=\{a_1,...,a_n,...\}$. \par

\begin{definition}[\cite{dung1995acceptability}]
An \textbf{abstract argumentation framework} (AF) is a directed graph $\naf=\ang{\args,\defrel}$ where $\args\subseteq\uni$ is a finite set of \emph{arguments} ($\args \subseteq \uni$) and $\defrel\subseteq \args\times \args$ is a \textit{defeat relation} among them. \\
\textbf{Notation}: We use `AFs' both as the plural of the abbreviation `AF' and as the name of the set of all AFs. The same convention is applied for the rest of the formalisms in this paper.
\end{definition}
% \begin{notation}
%     We use
% \end{notation}
The informal notion of acceptable argument is formally captured in AFs through different argumentation semantics (see, e.g, \cite{baroni2018abstract}). However, since they are orthogonal to our purposes, we skip their definition.\par 

Both arguments and defeats of an AF can be treated as uncertain. When we only focus on defeat-uncertainty, the standard tool to model this is defeat-IAFs.

\begin{definition}[\cite{cayrol2007partial,att-IAFs-2015-acceptance}]
A \textbf{defeat-incomplete abstract argumentation framework} (def-IAF) is a tuple $\ndefiaf=\ang{\args,\defrel^{\fix},\defrel^?}$ where $\args\subseteq \uni$ is a finite set of arguments and $\defrel^\fix,\defrel^?\subseteq \args\times \args$ are two disjoint sets of defeats, i.e.\ fixed and uncertain ones. A \textbf{completion} of $\ndefiaf=\ang{\args,\defrel^{\fix},\defrel^?}$ is any AF $\ang{\args,\defrel^{\ast}}$ s.th.
$\defrel^{F}\subseteq \defrel^\ast \subseteq (\defrel^{F}\cup \defrel^{?})$. We use $\compfun(\ndefiaf)$ to denote the set of all completions of $\ndefiaf$, and use the same notation for the formalisms below.
\end{definition}

In the previous notion of completion, all combinations of uncertain defeats are possible. This is why (def-)IAFs have been described as a \textit{combinatorial} way of modelling uncertainty \cite{jlc}. However, this kind of uncertainty can be shown to be very limited when one wants to model real argumentative scenarios. The following extension of def-IAFs overcomes this limitation.

\begin{definition}[\cite{fazzingakr21}]
Given a defeat-incomplete abstract argumentation framework $\ndefiaf=\ang{\args,\defrel^{\fix},\defrel^{?}}$, a \textbf{dependency} over $\ndefiaf$ is an expression of the form $\impor(X,Y)$, $\opor(X)$ or $\opnand (X)$, where $X$ and $Y$ are (non-empty) subsets of $\defrel^?$.\par 
A \textbf{defeat-incomplete abstract argumentation framework with dependencies} (dep-IAF) is a tuple $\ndepdefiaf=\ang{\args,\defrel^{\fix},\defrel^{?},\Delta}$ where $\ang{\args,\defrel^{\fix},\defrel^{?}}$ is an argument-incomplete abstract argumentation framework and $\Delta$ is a set of \emph{dependencies} over $\ang{\args,\defrel^{\fix},\defrel^{?}}$.
A \textbf{completion} of $\ang{\args,\defrel^{\fix},\defrel^{?},\Delta}$ is any AF $\ang{\args,\defrel^{\ast}}$ s.th.:
\begin{itemize}
\item $\ang{\args,\defrel^{\ast}}$ is a completion of $\ang{\args,\defrel^{\fix},\defrel^{?}}$; and
\item $\comp$ satisfies $\Delta$, that is, for all $\delta \in \Delta$:

\begin{itemize}
    \item If $\delta=\impor(X,Y)$, then $X\subseteq \defrel^{\ast}$ implies $Y \cap \defrel^{\ast}\neq \varnothing$.
    \item If $\delta=\opor(X)$, then $X\cap \defrel^{\ast}\neq \varnothing$.
     \item If $\delta=\opnand(X)$, then $X \cap \defrel^\ast \subset X$.
\end{itemize}
\end{itemize}

A few subclasses of dep-IAFs will be of interest to us. We say that a dep-IAF $\ndepdefiaf=\ang{\args,\defrel^{\fix},\defrel^{?},\Delta}$ is:

\begin{itemize}
    \item a \textbf{disjunctive def-IAF} (dis-IAF) iff $\Delta$ only contains constraints of the form $\opor(\{\ang{x,y},\ang{y,x}\})$ (recall that $\ang{x,y},\ang{y,x} \in \defrel^?$); 
    \item  an \textbf{implicative def-IAF} (imp-IAF) iff $\Delta$ only contains constraints of the form $\impor(X,Y)$ where $X$ and $Y$ are singleton. %    Again, the notation for $\impor(X,\{y\})$ is simplified as $\opimp(X,y)$; and

        \item a \textbf{disjunctive-implicative def-IAF} (dis-imp-IAF) iff $\Delta$ only contains constraints of any of the two forms described in the previous points.
\end{itemize}

\end{definition}

Note that disjunctive dependencies are enough to capture the static part of control AFs \cite{dimopoulos2018control}, another popular framework for arguing with qualitative uncertainty.

%%%%======================
\subsection{ASPIC$^+$}
%%%%======================

We now provide the basics of ASPIC$^{+}$, our baseline structured formalism. For discussion and motivation, the reader is referred to \cite{modgil2013aspic,modgil2014tutorial}.
\begin{definition}[\cite{modgil2013aspic}]\label{def:aspic-theory}
    An \textbf{argumentation theory} is a tuple $\at=\ang{\lanset, \negfun, \rulset, \namefun,\kb}$ where:
\begin{itemize}
\item $\lanset$ is a formal language.
\item $\negfun:\lanset \to \wp(\lanset)$ is a contrariness function. We say that:
\begin{itemize}
\item $\varphi$ is a contrary of $\psi$ iff $\varphi \in \overline{\psi}$ but $\psi \notin \overline{\varphi}$.
\item $\varphi$ is a contradictory of $\psi$ iff $\varphi \in \overline{\psi}$ and $\psi \in \overline{\varphi}$.
\end{itemize}
It is assumed that each $\varphi \in \lanset$ has at least one contradictory, denoted $-\varphi$.

\item $\rulset=\rulset_s \cup \rulset_d$ with $\rulset_s \cap \rulset_d =\varnothing$ is a set of inference rules (pairs of finite sets of formulas and formulas). %\footnote{ Given a set $S$, we denote by $\wp_{fin}(S)$ the set of all its finite subsets.} 
$\rulset_s$ represents strict rules while $\rulset_d$ represents defeasible rules.

\item $\namefun:\rulset_d \to \lanset$ is a (possibly) partial naming function for defeasible rules.
\item $\kb \subseteq \lanset$ is a knowledge base which comes split into two disjoint subsets $\kb_n$ (axioms) and $\kb_p$ (ordinary premises).
\end{itemize}
\end{definition}

The central notion of argument is then defined as follows:
\begin{definition}[\cite{modgil2013aspic}]
The set of \textbf{arguments of a given argumentation theory} $\at=\unravelat$, denoted $\args(\unravelat)$, is defined inductively. Together with the notion of argument, we define some auxiliary functions: $\sub(\cdot)$ (returns the \textbf{subarguments} of an argument), $\prem(\cdot)$ (returns the \textbf{premises} of an argument), $\conc(\cdot)$ (returns the \textbf{conclusion} of an argument), and $\ftoprule(\cdot)$ (returns the \textbf{last rule} employed in the construction of an argument). We establish that $A \in \args(\unravelat)$ iff $A$ is any finite expression built by the application of the following rules:

\begin{itemize}
\item $A=[\varphi]$ if $\varphi \in \kb$, with 
\begin{itemize}
    \item 
$\prem(A)=\conc(A)=\{\varphi\}$, 
\item 
$\sub(A)=\{[\varphi]\}$, and 
\item $\ftoprule(A)$ is left undefined.

\end{itemize}
\item $A=[A_1,...,A_n \sto \varphi]$ (with $n\geq 0$) if $A_1,...,A_n$ are arguments and \\ $\ang{\{\conc(A_1),...,\conc(A_n)\},\varphi} \in \rulset_s$, with 

\begin{itemize}
\item $\prem(A)=\prem(A_1)\cup...\cup\prem(A_n)$, 

\item $\conc(A)=\varphi$, 

\item $\sub(A)=\{A\}\cup \sub(A_1)\cup...\cup\,\sub(A_n)$, 

\item $\ftoprule(A)=\ang{\{\conc(A_1),...,\conc(A_n)\},\varphi}$.

\end{itemize}
\item $A=[A_1,...,A_n \Rightarrow \varphi]$ (with $n\geq 0$) if $A_1,...,A_n$ are arguments and \\ $\ang{\{\conc(A_1),...,\conc(A_n)\},\varphi} \in \rulset_d$, with:

\begin{itemize}
\item $\prem(A)=\prem(A_1)\cup...\cup\prem(A_n)$, 

\item $\conc(A)=\varphi$, 

\item $\sub(A)=\{A\}\cup \sub(A_1)\cup...\cup\,\sub(A_n)$, 

\item $\ftoprule(A)=\ang{\{\conc(A_1),...,\conc(A_n)\},\varphi}$.

\end{itemize}
\end{itemize}

We restrict our attention to theories that generate a \textbf{finite} set of arguments. When writting down arguments, squared brackets are omitted whenever no confusion arises.
% \noindent We omit squared brackets whenever no ambiguity arises. Note that arguments do not depend on $\contfun$ nor on $\namefun$, so we sometimes abbreviate $\args(\unravelat)$ as $\args(\lanset,\rulset,\kb)$.  Given $A \in \args (\at)$ we define the \textbf{rules of $A$} as $\rulset(A)=\{\ftoprule(B)\mid B \in \sub(A)\}$.  \par \medskip
\end{definition}

 \begin{remark}[Rules and arguments without premises]\label{remark:premise-less-rules} The previous definitions allow for (i) rules with an empty set of formulas on the left-hand side (i.e., it is possible that $\ang{\{\},\varphi} \in \rulset$); and consequently (ii) arguments with an empty set of premises ($\Rightarrow \varphi$ and $\sto\varphi$, but also more complex ones like $\sto \varphi, \Rightarrow \psi \Rightarrow \delta$). 
 %If an argument has no premise, we say it is \textbf{premiseless}. We say that an argument is \textbf{simple} iff it is either a single formula (a premise) or an argument of form $\Rrightarrow\varphi$ with $\Rrightarrow\in \{\sto,\Rightarrow\}$. 
 \end{remark}
     
     %  \begin{remark}[Excluding infinite arguments] The original definition of ASPIC$^+$ arguments allows for infinite arguments \cite{modgil2013aspic}: arguments which are rooted in a conclusion but keep growing infinitely without ``finding'' a set of premises. These abnormal arguments are later excluded from the definition of associated AF in the ASPIC literature. We exclude them from the very definition of argument for presentational purposes.         
     % \end{remark}
     \black
Next, one needs to codify a notion of conflict among arguments.
\begin{definition}[\cite{modgil2013aspic}]
Given an argumentation theory $\at=\unravelat$, and two arguments $A,B\in \args(\at)$, we say that \textbf{$A$ attacks $B$} iff $A$ undermines, rebuts or undercuts $B$, where: 

\begin{itemize}
\item $A$ \textbf{undermines} $B$ (on $B'$) iff $\conc(A)\in \overline{\varphi}$ for some $B'=\varphi \in \prem(B)$ such that $\varphi \in \kb_p$. {We say that $A$ \textbf{contrary-undermines} $B$ if $\conc(A)$ is a contrary of $\varphi$.}
\item $A$ \textbf{rebuts} $B$ (on $B'$) iff $\conc(A)\in \overline{\varphi}$ for some $B' \in \sub(B)$ of the form $B_1',....,B_n'\Rightarrow\varphi$. {We say that $A$ \textbf{contrary-rebuts }$B$ if $\conc(A)$ is a contrary of $\varphi$.}

\item $A$ \textbf{undercuts} $B$ (on $B'$) iff $\conc(A)\in \overline{\namefun(\ftoprule(B'))}$ for some $B' \in \sub(B)$ with $\ftoprule(B')\in \rulset_{d}$.
\end{itemize}

We define two shorthands that will be useful later on:

\begin{itemize}
    \item We say that there is a \textbf{preference-independent attack} of $A$ to $B$ (on $B'$) iff $A$ either undercuts, contrary-rebuts or contrary-undermines $B$ (on $B'$). We shorten this as $\pia(A,B,B')$.
    \item We say that there is a \textbf{preference-dependent attack} from $A$ to $B$ (on $B'$) in the rest of the cases of the previous definition. We abbreviate this as $\pda(A,B,B')$. 
\end{itemize}

\end{definition}

%A simple conflict is sometimes not enough to decide whether an argument wins over another. To this end, ASPIC$^{+}$ adds a preference relation among arguments, which is integrated in the formalism as follows:

%Preference-dependent attacks are symmetric. 
Some of these attacks can be ignored if we have a preference relation among arguments. Intuitively, a preference relation encodes some sort of epistemic precedence of certain arguments over others. Formally, it results in the following definition.

\begin{definition}[\cite{modgil2013aspic}]
A \textbf{structured argumentation framework} (SAF) is a tuple $\nsaf=\unravelsaf$ where $\at=\unravelat$ is an argumentation theory and $\leq$ is a preferential ordering relation among arguments $\leq \subseteq \args(\at)\times \args(\at)$ (its strict counter-part of $\leq$, denoted $<$, is defined as usual: $<=\leq\setminus \leq^{-1}$).
\end{definition}
Preference relations are usually assumed to be a partial preorder, but we take the fully general notion (as in \cite{modgil2013aspic}).
Furthermore, in many contexts and applications, preferences are usually defined among formulas and/or rules and then lifted to arguments. Different lifting principles that respect rationality postulates (i.e., desirable properties about the framework output) have been studied (see \cite{modgil2013aspic}).
We stay in a semi-abstract perspective here, in order not to commit to any particular kind of preference.

\begin{definition}[\cite{modgil2013aspic}]
Given $\nsaf=\unravelsaf$, and $A,B\in \args(\unravelat)$, we say that $A$ \textbf{defeats} $B$ iff:
\begin{itemize}
 \item[(i)]  $A$ undercuts/contrary-rebuts/contrary-undermines $B$;\footnote{Equivalently, $\pia(A,B,B')$ for some $B'$.} \black or 
 \item[(ii)] $A$ undermines/rebuts $B$ (on $B'$) and $A\not < B'$.\footnote{Equivalently, $\pda(A,B,B')$ and $A\not < B'$ for some $B'$.} 
\end{itemize}

The set of all defeats for a given $\nsaf$ is denoted $\defrel(\nsaf)$. We instead equate the set of arguments $\args(\nsaf)$ of a given $\nsaf$ with those of its underlying argumentation theory, i.e. $\args(\nsaf) = \args(\at)$. In general, given $\nsaf=\unravelsaf$, we use $\lanset(\nsaf)$ to denote $\lanset$ and apply the same convention for the rest of the components (including the non-primitive components $\args(\nsaf)$ and $\rel(\nsaf)$). 
\end{definition}

The following definition provides the link between abstract and structured models of argumentation.%key notion that lifts structured argumentation frameworks to the abstract level by identifying their associated graph.
\begin{definition}[\cite{modgil2013aspic}]
   Let $\nsaf=\unravelsaf$ be given, the \textbf{abstract argumentation framework associated to $\nsaf$} is defined as $\daf(\nsaf)=\ang{\args(\nsaf),\defrel(\nsaf)}$. 
\end{definition}

\begin{example}\label{example:saf} Consider $\nsaf_0=\indexedsaf{}$, where:
\begin{itemize}
    \item $\lanset$ is the language of propositional logic; 
    \item $\overline{\cdot}$ is given by classical negation (i.e., $\varphi \in \overline{\psi}$ iff $\varphi=\lnot \psi$ or $\psi=\lnot \varphi$);
    \item $\rulset_s=\{\ang{\{u\},\lnot s}\}$;
    \item $\rulset_d=\{\ang{\{p\},q}, \ang{\{w\},r}, \ang{\{s\},\lnot r} \}$;
    \item  $\namefun$ is only defined for $\namefun(\ang{\{p\},q})=r$;
    \item $\kb_n=\{p,u\}$;
    \item $\kb_p=\{s, w\}$;
    \item $\leq =\{\ang{ s \Rightarrow \lnot r, w \Rightarrow r }\} $.
\end{itemize}

The associated AF looks as follows, where each box is an argument and arrows represent the defeat relation:

\begin{center}
\begin{tikzpicture}[->,>=stealth,shorten >=1pt,auto,node distance=1.4cm,
                thick,main node/.style={circle,draw,font=\bfseries},uncertain/.style={rectangle,draw,dashed,font=\bfseries}]

\node[sarg] (s) {$s$};
\node[sarg] (snr) [right=0.5cm of s] {$s\Rightarrow \lnot r$};

\node[sarg] (pq) [right=2cm of snr] {$p \Rightarrow q$};
\node[sarg] (p) [above=0.1cm  of pq] {$p$};

\node[sarg] (uns) [above=0.6cm of s] {$u\sto \lnot s$};
\node[sarg] (u) [above=0.1cm of uns] {$u$};

\node[sarg] (wr) [right=2cm of uns] {$w\Rightarrow r$};
\node[sarg] (w) [above=0.1cm  of wr] {$w$};

\path[->] (snr) edge (pq);
\path[->] (uns) edge (snr);
\path[->] (uns) edge (s);
\path[->] (wr) edge (snr);
\end{tikzpicture}
\end{center}
    
\end{example}

%=========================================================
\section{ASPIC$^+$ with uncertain preferences}
%=========================================================

This section contains the main contributions of the paper. 

\newcommand{\prefcomp}{\ang{\lanset, \negfun, \rulset, \namefun, \kb, \leq^{\ast}}}

%%%%%%==================
\subsection{Definitions and examples}
%%%%%%==================

Let us start by defining our object of study. As mentioned, preference-incomplete structured argumentation frameworks were first discussed by \cite{baumeister2021acceptance} as a plausible instantiation of def-IAFs. They were first defined by \cite{ria2024}.

\begin{definition}[\cite{ria2024}]\label{def:pref-ISAFs}    
A \textbf{preference-incomplete structured argumentation framework} (pref-ISAF) is a tuple $\mathsf{IS}=\unravelsaf$,
where every component is just as in a SAF except for $\leq$, which comes split into two parts $\leq=\leq^{F}\cup \leq^{?}$ with $\leq^{F}\cap\leq^{?}=\varnothing$.
\end{definition}
%\todo[inline]{provide intuition (e.g., multi-agency)}

Intuitively, $\leq^\fix$ represents the known or certain preference relation, while $\leq^?$ is the uncertain part. This uncertainty can be understood as stemming from multi‑agency. In this picture, a pref-ISAF is a model of how an agent $i$ sees the argumentative situation of another agent $j$. However, agent $i$ is not sure about how $j$ evaluates the relevant arguments preference-wise, and hence the preference $\leq$ is uncertain in the model.

\begin{definition}\label{def:pref-ISAFs-completions}  
A \textbf{preference-completion of $\nisaf=\unravelsaf$} is any $\nsaf^{\ast}=\prefcomp$ where $\leq^{F}\subseteq \leq^{\ast}\subseteq (\leq^{F}\cup\leq^{?})$. We denote as $\prefprefix\compfun(\nisaf)$ the set of all preference-completions of $\mathsf{IS}$. 
Given $\nsaf^\ast=\prefcomp\in \prefprefix\compfun(\nisaf)$, we abbreviate $\defrel(\prefcomp)$ as $\defrel(\leq^\ast)$ when the context is clear.
\par 
 Let $\mathsf{IS}$ be given, its set of (abstract) \textbf{completions} is defined as:

\centerline{$\compfun(\mathsf{IS})=\{\daf(\nsaf^{\ast})\mid \nsaf^{\ast} \in \prefprefix\compfun(\nisaf)\}\text{.}$}

\end{definition}

Let us now see a couple of examples that will be useful later on.

\begin{example}\label{example:pref-ISAF} Consider $\mathsf{IS}_0=\unravelsaf$, where:
\begin{itemize}
    \item $\lanset$ is the language of propositional logic; 
    \item $\overline{\cdot}$ is given by classical negation (i.e., $\varphi \in \overline{\psi}$ iff $\varphi=\lnot \psi$ or $\psi=\lnot \varphi$);
    \item $\rulset_d=\{\ang{\{p\},r}, \ang{\{q\},\lnot r}\}$;
    \item $\kb_p=\{p, q\}$;
    \item $\leq^?= \{\ang{p \Rightarrow r, q \Rightarrow \lnot r}, \ang{q \Rightarrow \lnot r,p \Rightarrow r}\} $.
    \item The rest of the components are empty.
\end{itemize}

Note that $\mathsf{IS}_0$ has four preference-completions, namely, the frameworks obtained by using the preference sets $\varnothing$,  $\{\ang{p \Rightarrow r, q \Rightarrow \lnot r}\} $, $\{\ang{q \Rightarrow \lnot r,p \Rightarrow r}\}$ and $\leq^?$. However, $\varnothing$ and $\leq^?$ collapse in the same associated framework, so we obtain three completions: \par 
\medskip
\scalebox{0.95}{
\begin{tabular}{c|c |c}

     \begin{tikzpicture}[->,>=stealth,shorten >=1pt,auto,node distance=1.4cm,
                thick,main node/.style={circle,draw,font=\bfseries},uncertain/.style={rectangle,draw,dashed,font=\bfseries}]

\node[sarg] (pr) {$p \Rightarrow r$};
\node[sarg] (p) [below=0.2cm of pr] {$p$};
\node[sarg] (qnr) [right=1cm of pr] {$q\Rightarrow \lnot r$};
\node[sarg] (q) [below=0.2cm of qnr] {$q$};

\path[<->] (pr) edge (qnr);

\end{tikzpicture}
\quad
     & 
\quad
     % SECOND COMPLETION
          \begin{tikzpicture}[->,>=stealth,shorten >=1pt,auto,node distance=1.4cm,
                thick,main node/.style={circle,draw,font=\bfseries},uncertain/.style={rectangle,draw,dashed,font=\bfseries}]

\node[sarg] (pr) {$p \Rightarrow r$};
\node[sarg] (p) [below=0.2cm of pr] {$p$};
\node[sarg] (qnr) [right=1cm of pr] {$q\Rightarrow \lnot r$};
\node[sarg] (q) [below=0.2cm of qnr] {$q$};

\path[->] (pr) edge (qnr);

\end{tikzpicture}
     \quad
     & 
         \quad
      % THIRD COMPLETION
   
       \begin{tikzpicture}[->,>=stealth,shorten >=1pt,auto,node distance=1.4cm,
                thick,main node/.style={circle,draw,font=\bfseries},uncertain/.style={rectangle,draw,dashed,font=\bfseries}]

\node[sarg] (pr) {$p \Rightarrow r$};
\node[sarg] (p) [below=0.2cm of pr] {$p$};
\node[sarg] (qnr) [right=1cm of pr] {$q\Rightarrow \lnot r$};
\node[sarg] (q) [below=0.2cm of qnr] {$q$};

\path[<-] (pr) edge (qnr);

\end{tikzpicture}
     \\
\end{tabular}
    }
\end{example}

\begin{example}\label{ex:isaf2} 
 %\todo[inline]{revise (contrary rebuts are pref-independent). put as a separate example}
Let $\nisaf_1= \unravelsaf$, where $\lanset$ and $\contfun$ are as in the previous example, and where
    $\kb_n=\rulset_s=\varnothing$; $\kb_p=\{p,\lnot p\}$, $\rulset_d=\{\ang{\{p\},q}\}$, $\leq^{F}=\{\ang{\lnot p,p}\}$, $\leq^{?}=\{\ang{p,\lnot p}\}$. We then obtain two preference-completions and two associated AFs, namely 

\begin{center}
\begin{tikzpicture}[->,>=stealth,shorten >=1pt,auto,node distance=1.4cm,
                thick,main node/.style={circle,draw,font=\bfseries},uncertain/.style={rectangle,draw,dashed,font=\bfseries}]
\node[sarg] (np) {$\lnot p$};
\node[sarg] (p) [right of=np] {$p$};
\node[sarg] (pq) [above of=p] {$p\Rightarrow q$};

\path[->] (np) edge[bend left] (p)
        (p) edge (np)
		(np) edge (pq);

\node[sarg] (np) [right=3cm of p] {$\lnot p$};
\node[sarg] (p) [right of=np] {$p$};
\node[sarg] (pq) [above of=p] {$p\Rightarrow q$};
\path[->] (p) edge (np);
\end{tikzpicture}
\end{center}

\end{example}

As mentioned, preference-uncertainty causes defeat-uncertainty at the abstract level. Let us make this claim more formal.
\begin{proposition}\label{prop:pref-ISAFs-def-incompleteness}
   Let $\mathsf{IS}$ be a pref-ISAF, and let $\ang{\args_1,\defrel_1},\ang{\args_2,\defrel_2}\in \compfun(\mathsf{IS})$. Then, $\args_1=\args_2$.
\end{proposition}

\begin{proof}
  $\ang{\args_1,\defrel_1}$ and $\ang{\args_2,\defrel_2}$ are the AFs associated with two preference-completions $\nsaf_1$ and $\nsaf_2$ that share their formal languages, sets of rules and knowledge bases; therefore, they must have the same arguments.
\end{proof}

%=============================================
\subsection{Comparing sets of completions}
%=============================================

Recall our research question: What kind of defeat-uncertainty do we get from preference-uncertainty? To provide a clear, formal answer, we need a way to compare sets of completions. The following definition does so.
\begin{definition}
    Given two sets of AFs $S=\{\ang{\args,\defrel_1},...,\ang{\args,\defrel_n}\}$ and $S'=\{\ang{\args',\defrel_1'},...,\ang{\args',\defrel'_n}\}$ with the same cardinality, such that each AF in $S$ (resp.\ $S'$) has the same set of arguments (i.e., given two sets of defeat-completions). We say that they are \textbf{equivalent} (in symbols, $S\approxeq S'$) iff $\ang{\args,\defrel_1,....,\defrel_n}$ is isomorphic to $\ang{\args',\defrel'_1,....,\defrel'_n}$.\par 

Given two classes of argumentative formalisms with qualitative uncertainty (i.e., def-IAFs, dep-IAFs, pref-ISAFs, etc) $\mathsf{X}$ and $\mathsf{Y}$, we say that \textbf{$\mathsf{X}$ is at least as expressive as $\mathsf{Y}$} (in symbols, $\mathsf{Y} \lexpressive \mathsf{X}$) iff for all $Y \in \mathsf{Y}$ there is a $X \in \mathsf{X}$ such that $\compfun(X) \approxeq \compfun(Y)$. %If the function $i$ is clear enough from context, we omit it and identify each $x$ with $i(x)$.
\end{definition}

\begin{fact} The relation $\lexpressive$ is reflexive and transitive. \end{fact}

\begin{remark} The previous notion of expressivity significantly simplifies the one used in \cite{proietti2025comparative} for comparing sets of completions. The reason for this simplification is that we focus here on defeat-uncertainty (vs.\ the argument-uncertainty of \cite{proietti2025comparative}), and hence we do not need to keep track of the identity of arguments across different completions to get an intuitive notion of equivalence.
\end{remark}

The following proposition, proved by \cite[p. 304]{fazzingaijcai21}, shows that dep-IAFs provide an upper bound for expressing uncertainty about defeats. In other words, dep-IAF are maximally expressive wrt defeat incompleteness.

\begin{proposition}[\cite{fazzingakr21}]\label{prop:dep-iaf} Let $T=\{\ang{\args,\defrel_1},...,\ang{\args,\defrel_n}\}$ be any finite set of AFs with the same domain (i.e., any finite set of defeat-completions). There is a dep-IAF $\ndepdefiaf=\ang{\args,\defrel^{\fix},\defrel^?,\Delta}$ such that $\compfun(\ndepdefiaf)=T$.
%\todo[inline]{sketch proof if we have space.}

%\begin{}
% \textbf{generalise and sketch argument}

% Let $\ndefiaf = \ang{\args,\defrel^{\fix},\defrel^?}$ be an def-IAF and $X \subseteq \compfun(\ndefiaf)$ a subset of its possible completions. There is always a set $\Delta$ of dependencies such that $\compfun(\ang{\args,\defrel^{\fix},\defrel^?,\Delta}) = X$.
\end{proposition}

Hence, as a corollary, dep-IAFs are expressive enough to subsume pref-ISAFs.

\begin{corollary}
    pref-ISAFs $\lexpressive$ dep-IAFs.
\end{corollary}

%================================================
\subsection{Some negative results}
%================================================

The next natural question is if we can get rid of dependencies and simulate pref-ISAFs at the abstract level with simple def-IAFs (as originally suggested by \cite{baumeister2021acceptance}). The answer was shown to be negative by \cite[Proposition 2]{ria2024}.

\begin{proposition}[\cite{ria2024}]\label{prop:negative-def-IAFs}
$\text{pref-ISAFs } \not \lexpressive \text{ def-IAFs.}$
\end{proposition}

\begin{proof}
    As a witness, consider the pref-ISAF of Example \ref{example:pref-ISAF}. Reasoning towards a contradiction, suppose that there is a def-IAF with an equivalent set of completions. Note that the defeats among $p \Rightarrow r$ and $q \Rightarrow \lnot r$ must be uncertain in the abstract set of completions (because they appear in some but not in some others of $\mathsf{IS}$), but then there must be an abstract completion containing no defeats, and therefore an equivalent one in the pref-ISAF, and this is not the case.
\end{proof}

It can be shown, however, that the pref-ISAF of Example \ref{example:pref-ISAF} is equivalent to a disjunctive def-IAFs. So our next question is whether this can be generalised to all finite pref-ISAFs. The answer, again, is negative.

\begin{proposition}\label{prop:2-negative-def-IAFs}
$\text{pref-ISAFs } \not \lexpressive \text{ dis-IAFs.}$
\end{proposition}

\begin{proof}
    
As a witness, consider the set of completions of the ISAF of Example \ref{ex:isaf2}.
It is easy to show that these AFs do not correspond to the set of completions of any dis-IAF.

\end{proof}

%Although the pref-ISAF used in the proof of the previous result corresponds, in an abstract perspective, to an imp-IAF, this is not the case in general.

Note that the pref-ISAF used in the previous proof encodes an implicative defeat dependency $\opimp(\{\ang{\lnot p,p}\},\{\ang{\lnot p, p\Rightarrow q}\})$. However, it also encodes a disjunctive dependency, namely $\opor(\{\ang{p,\lnot p},\ang{\lnot p, p}\})$. Hence, we can show that implicative dependencies alone are not enough to capture pref-ISAFs. This contrasts with recent work on the expressivity of uncertain rules and premises, where they are shown to be captured by implicative argument-incomplete AFs \cite{proietti2025comparative}.

\begin{proposition}\label{pref-ISAF-not-lespressive-imp-iafs}
    $\text{pref-ISAFs } \not \lexpressive \text{ imp-IAFs}$
\end{proposition}

\begin{proof}
    The pref-ISAF of Example \ref{example:pref-ISAF} serves again as a witness. Details are omitted for brevity.
\end{proof}

We can ask the reverse question: Are pref-ISAFs expressive enough to simulate the kind of defeat-uncertainty encoded in def-IAFs? A negative answer is provided in the next proposition. Again, this can be compared with recent work on the expressivity of uncertain rules and premises, where they are shown to be more expressive than abstract argument-incomplete AFs \cite{proietti2025comparative}. 
\begin{proposition}
    $\text{def-IAFs } \not \lexpressive \text{ pref-ISAFs}$.
\end{proposition}

\begin{proof} %\todo[inline]{revise}
    Take $\ndefiaf=\ang{\args,\defrel^\fix,\defrel^?}$ where $\args=\{a\}$, $\defrel^\fix=\varnothing$ and $\defrel^?=\{\ang{a,a}\}$. We have two completions, namely, $\naf_1=\ang{\{a\}, \varnothing}$ and $\naf_2=\ang{\{a\}, \{\ang{a,a}\}}$. Reasoning toward contradiction, suppose there is a pref-ISAF $\mathsf{IS}=\unravelsaf$ with an equivalent set of completions. Let $i: \args \to \args(\unravelsaf) $ be an isomorphism that witnesses this equivalence. Since we only have one argument, then either $i(a)$ is the only element of $\kb_p$, call it $\varphi$, or $i(a)$ is of the form $\Rightarrow \varphi$ (a premise-less argument with a defeasible rule). Otherwise, we would have more than one argument or an undefeatable one, and $\naf_1$ could not be generated (the reader is invited to check these details). We can also deduce that $\varphi=-\varphi$ ($\varphi$ is a contradictory formula). But then, the value of $\leq$ does not matter, since in both cases we have that $i(a)$ defeats $i(a)$; and hence no structured completion is equivalent to $\ang{\{a\}, \varnothing}$. \par 

\end{proof}

The following corollaries follow from the previous proposition, the definition of $\lexpressive$ and the fact that every def-IAF can be seen as a (subtype of) dep-IAFs with an empty set of dependencies.
\begin{corollary}
      $\text{dis-IAFs } \not \lexpressive \text{ pref-ISAFs}$.
\end{corollary}

\begin{corollary}
    $\text{dis-imp-IAFs } \not \lexpressive \text{ pref-ISAFs}$.
\end{corollary}

\begin{corollary}
    dep-IAFs  $\not \lexpressive$ pref-ISAFs.
\end{corollary}

%=========================================
\subsection{Towards a non-trivial positive expressivity threshold}
%=========================================
Finally, we take the first step towards providing a non-trivial, upper threshold for the expressivity of pref-ISAFs. Let us provide some definitions.

\begin{definition}[Candidate relations and dependencies]
    Let $\nisaf=\unravelsaf$ be a pref-ISAF with $\leq=\leq^\fix\cup \leq^?$. Define the following relations and sets of dependencies over $\args(\nisaf)$:

    \begin{itemize}
    \item $\leq^N= (\args(\nisaf)\times \args(\nisaf)) \setminus \leq$.
    
        \item $\defrel^\fix_\nisaf=\{\ang{X,Y}\mid\exists Y' \pia(X,Y,Y')\lor \pda(X,Y,X)\lor  $\\ \indent  $\qquad \qquad  \qquad \exists Y'(\pda(X,Y,Y')\land (X\leq^NY'\lor Y'\leq^\fix X))\}$.
        
        \item $\defrel^?_\nisaf=\{\ang{X,Y}\mid\ang{X,Y}\notin \defrel^\fix_\nisaf \land \exists Y'\big( \pda(X,Y,Y')\land (X\leq^?Y'\lor Y'\leq^? X)\big)\}$.
\item    {\small $\Delta^\opor_\nisaf=\{\opor(\{\ang{X,Y},\ang{Y,X}\})\mid \pda(X,Y,Y)\land\pda(Y,X,X)\land X\leq^? Y\land Y \leq^? X \}$.}
\item $\Delta^\opimp_\nisaf=\{\opimp(\{\ang{X,Y}\},\{\ang{X,Y'}\})\mid Y \in \sub(Y')\}$.
    \end{itemize}
\end{definition}

\begin{proposition}\label{prop:main}
Let $\nisaf=\unravelsaf$ be a pref-ISAF. Then for all $X,Y,W \in \args(\nisaf)$, we have that:

\begin{enumerate}
\item $\ang{X,Y}\in \defrel^\fix$ iff $\forall \comp \in \compfun(\nisaf)$, $\ang{X,Y}\in \defrel^\ast$.

    \item $(\exists \comp\in \compfun(\nisaf)$ such that $\ang{X,Y}\in \defrel^\ast$) iff $\big(\ang{X,Y}\in \defrel^\fix$ OR $\ang{X,Y}\in \defrel^?_\nisaf\big )$.
    
    \item $\opor({\ang{X,Y},\ang{Y,X}})\in \Delta^\opor_\nisaf$ implies $\forall \comp \in \compfun(\nisaf)$ either $\ang{X,Y}\in \defrel^\ast$ or $\ang{Y,X}\in \defrel^\ast$.
    \item $\opimp(\ang{X,Y},\ang{X,W})\in \Delta^\opimp_\nisaf$ implies that $\forall \comp \in \compfun(\nisaf)$, if $\ang{X,Y}\in \defrel^\ast$, then $\ang{X,W}\in \defrel^\ast$.
\end{enumerate}
\end{proposition}

\begin{proof} (\textbf{1.}) ($\Longrightarrow$) Suppose $\ang{X,Y}\in\defrel_\nisaf^\fix$. Let $\ang{X,Y}\in \defrel^\ast$ and $\defrel^\ast=\defrel(\leq^\ast)$ ($\leq^\ast$ is the preference relation of the preference-completion that generates $\defrel^*)$. We have three cases (on the membership to $\defrel^\fix_\nisaf$).\par  
   \noindent (\textbf{Case}: $\exists Y'(\pia(X,Y,Y'))$) Then it does not matter the preference relation, we will have $\ang{X,Y}\in \defrel(\leq^\ast)$. \par 
      \noindent (\textbf{Case}: $\pda(X,Y,X)$) Suppose, reasoning towards contradiction, that $\ang{X,Y}\notin\defrel(\leq^\ast)$, but then by the case hypothesis and the definition of defeat, we have $X\leq^\ast X$ and $X\nleq^\ast X$, which is absurd. \par 
\noindent (\textbf{Case}: $\exists Y'(\pda(X,Y,Y')\land (X\leq^N Y'\lor Y'\leq^\fix X))$) Let $Y_1$ be such an $Y'$. We have two cases, if $X\leq^NY_1$, then $X\nleq^\ast Y$ (and hence $\ang{X,Y}\in \defrel(\leq^\ast)$, because $\pda(X,Y,Y_1))$. If, on the other hand, $Y_1\leq^\fix X$, then $Y_1\leq^\ast X$ and we arrive to  $\ang{X,Y}\in \defrel(\leq^\ast)$ again.
\par 

(\textbf{1.}) ($\Longleftarrow$)
Suppose that:
\\
\noindent (1) $\forall \comp \in \compfun(\nisaf)$, $\ang{X,Y}\in \defrel^\ast$, \\
\noindent (2) $\forall Y'(\lnot \pia(X,Y,Y'))$, \\
\noindent (3) $\lnot \pda(X,Y,X)$, \\ and, reasoning towards contradiction, suppose that: \\
 \noindent (4) $\forall Y'(\pda(X,Y,Y')\to (X \nleq^N Y'\land Y'\nleq^\fix X)$.\\
 Moreover, define $\leq_1=\leq^\fix\cup\{\ang{X,Y'}\mid \pda(X,Y,Y')\}$. We will show that $\ang{\args,\defrel(\leq_1)}\in \compfun(\nisaf)$ and $\ang{X,Y}\notin \defrel(\leq_1)$ (which contradicts (1)). For the former claim, note that $\leq^\fix\subseteq \leq_1\subseteq \leq^\fix\cup\leq^?$ follows from the definition of $\leq_1$ ((4) is needed to deduce this, details are omitted for brevity). For the latter claim, by (2), we know that if $\ang{X,Y}\in \defrel(\leq^\ast)$, then there must be an argument $Z$ such that $\pda(X,Y,Z)$ and $X\not<_1Z$. But this can be shown to be absurd. For suppose there is such a $Z$, then it follows from the definition of $\leq_1$ that $X\leq_1 Z$. Then it must be the case that $Z \leq_1 X$ and, from (4) and $\pda(X,Y,Z)$, we can derive $Z\leq^?X$, which by definition of $\leq_1$ leads to $\pda(X,Y,X)$, which contradicts (3).\par 
\par 
    (\textbf{2.}) ($\Longrightarrow$) Suppose $\ang{X,Y}\in \defrel^\ast$ for some $\defrel^\ast\in \compfun(\nisaf)$. Let $\defrel^\ast=\defrel(\leq^\ast)$. This implies by definition of defeat that $\exists Y'(\pia(X,Y,Y'))\lor\exists Y'(\pda(X,Y,Y')\land X \not <^\ast Y') $. We continue by cases.

    \noindent (\textbf{Case}: $\exists Y'(\pia(X,Y,Y'))$) Then $\ang{X,Y}\in \defrel^\fix_\nisaf$ by definition and we are done. 

   \noindent (\textbf{Case}: $\exists Y'\pda(X,Y,Y')\land X \not <^\ast Y')$) Let $Z=Y'$, we get $\pda(X,Y,Z)$ and ($X\not \leq^\ast Z$ or $Z \leq^\ast X$). The second claim is equivalent to a four-element disjunction: either $X\leq^NZ$ or $X\leq^?Z$ or $Z\leq^F X$ or $Z\leq^? X$ (by definition of preference-completions). Now it can be shown, on the one hand, that $X\leq^N Z$ or $X \leq^F Z$ leads to $\ang{X,Y}\in \defrel^\fix_\nisaf$. On the other hand, $X\leq^? Z$ or $Z\leq^? X$ leads to $\ang{X,Y}\in \defrel^\fix_\nisaf$ or $\ang{X,Y}\in \defrel^?_\nisaf$ and, retrieving the existential quantifier, we are done. Details are omitted for brevity.

   \par 
   (\textbf{2.}) ($\Longleftarrow$) Suppose that $\ang{X,Y}\in \defrel^\fix_\nisaf$ or $\ang{X,Y}\in \defrel^?_\nisaf$. \par 

   \noindent (\textbf{Case}: $\ang{X,Y}\in \defrel^\fix_\nisaf$) It follows from item 1 of this proposition. \par 

   \noindent(\textbf{Case}: $\ang{X,Y}\in \defrel^?_\nisaf$). Then $\exists Y'(\pda(X,Y,Y')\land (Y'\leq^? X \lor X \leq^?Y'))$. Let $Y_1$ be such an $Y'$, then either $Y_1\leq^? X$ or $X\leq^?Y_1$. For the first case ($Y_1\leq^? X$), define $\leq_1=\leq^\fix\cup\{(Y_1,X)\}$. It is then easy to show that $\ang{\args,\defrel(\leq_1)}\in \compfun(\nisaf)$, and that $\ang{X,Y}\in \defrel(\leq_1)$. For the second case ($X\leq^?Y_1$), we can arrive that $\ang{X,Y}\in \defrel(\leq^\fix)$.

   \par 
\par 

 (\textbf{3.}) Let $\opor(\{\ang{X,Y},\ang{Y,X}\})\in \Delta^\opor_\nisaf$, which implies $\pda(X,Y,Y)$, $\pda(Y,X,X)$, $X\leq^? Y$ and $Y\leq^?X$. Let $\comp \in \compfun(\nisaf)$, with $\defrel^\ast=\defrel(\leq^\ast)$. Suppose that $\ang{X,Y}\notin \defrel(\leq^\ast)$. The latter implies together with $\pda(X,Y,Y)$ and the definition of defeat that $X<^\ast Y$, which together with $\pda(Y,X,X)$ implies $\ang{Y,X}\in \defrel(\leq^\ast)$.

   (\textbf{4.}) This is a well known fact of ASPIC$^+$ (it follows directly from the definition of defeat).
\end{proof}

\paragraph{Discussion} %If we could strengthen the previous Proposition to get the right-to-left directions of items 1, 3 and 4, we would get a positive, non-trivial threshold for the abstract expressivity of pref-ISAFs. So far, the problem with point 1 is the phenomenon of polymorphic attacks in ASPIC$^+$\cite{baroni2025polymorphic}: the fact that there can be several (preference-dependent) attacks from $X$ to $Y$, which creates a puzzle when one is to construct the right preference completion to show that the attack does not succeed in every preference-completion. 
%However, we are confident that the proof can be found. 
If we could strengthen points 3 and 4 of the previous proposition to show that any other possible dependency that $\nisaf$ satisfies is already implied by $\Delta^\opor_\nisaf$ and $\Delta^\opimp_\nisaf$, then a non-trivial abstract expressivity threshold for pref-ISAFs would have been reached. We have not arrived there yet. Nonetheless, after examining a substantial number of potential counterexamples without success, we formulate the following conjecture, which delineates our main venue for future work.

\begin{conjecture}\label{conjecture:pref-ISAFs}
pref-ISAFs $\lexpressive$ dis-imp-IAFs.
\end{conjecture}

%\missingfigure[]{expressivity results}

%========================
\section{Conclusion}
%========================

\paragraph{Recap.} The following figure sums up all the results of the paper where a directed arrow from $\mathsf{X}$ to $\mathsf{Y}$ represents that $\mathsf{Y}$ is strictly more expressive than $\mathsf{X}$. Transitive arrows have been omitted for readability. Moreover, our main conjecture is represented through a dashed arrow.

\begin{center}
    
\includegraphics[scale=1]{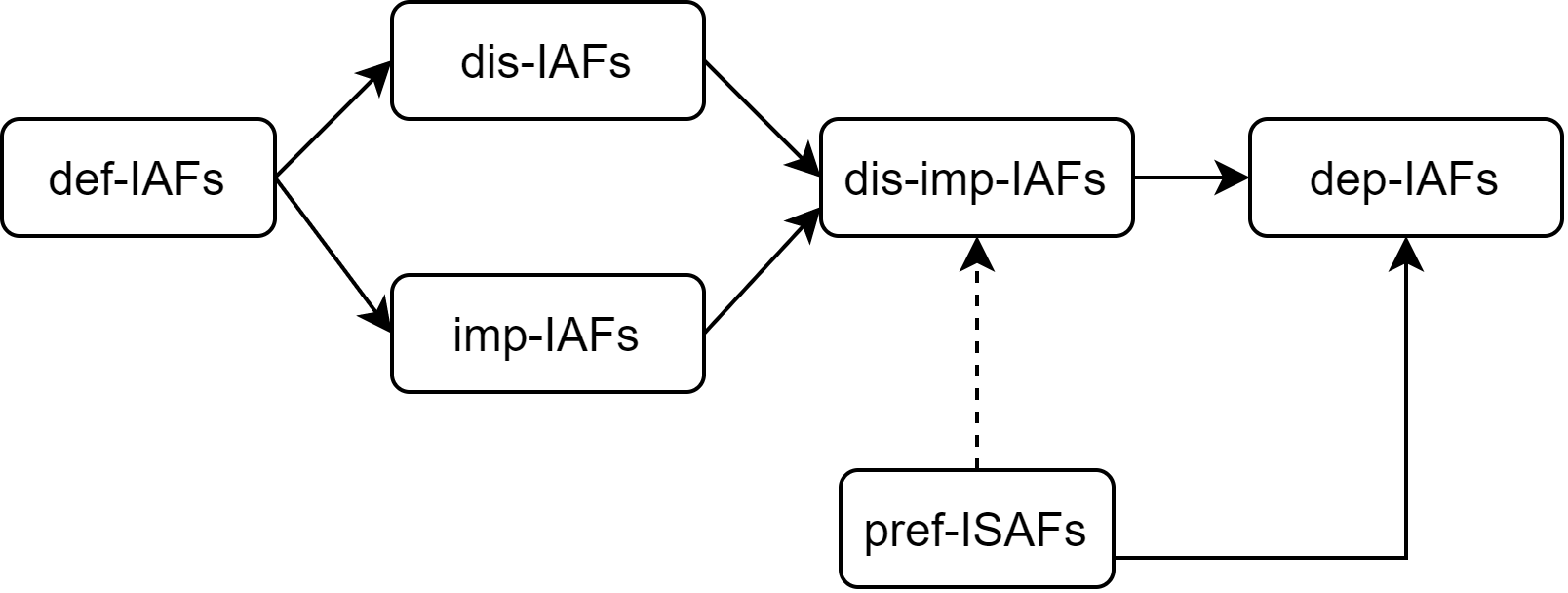}

\end{center}

\paragraph{Future work.} We close this paper by outlining three clear open paths for future work. First, and foremost, we aim at proving Conjecture \ref{conjecture:pref-ISAFs}. This, besides informing us with what kind of expressivity is needed for modelling uncertain preferences at the abstract level, would also give us some hints about computational complexity, since dis-imp-IAFs are a well-studied fragment of dep-IAFs (see \cite[Table 1]{fazzingakr21}). Second, as mentioned in the Background section, preferences among arguments in ASPIC$^+$ are sometimes rooted on more basic relations among defeasible rules and/or ordinary premises, and these are in turn assumed to have some properties. Hence, a natural way to continue our work is to study uncertainty at this more concrete level, and to compare it with the current results on uncertain preferences at the level of arguments. Finally, we could compare the expressivity of pre-ISAFs with other forms of uncertainty that generate uncertain defeats (namely, uncertain contrary functions and uncertain naming conventions).

% \begin{itemize}
% \item Proving the conjecture (perhaps for a strict fragment \cite{odekerken2025argumentative}).
% \item Uncertainty in a more concrete level of preferences (among premises or rules).
% \item Comparison with other form of structured, defeat-uncertain formalisms.
% \end{itemize}

\bibliographystyle{plain}
\bibliography{argpi_biblio}
\end{document}